\documentclass[american,english, letterpaper, 10 pt, conference]{ieeeconf}
\IEEEoverridecommandlockouts
\usepackage[T1]{fontenc}
\usepackage[latin9]{inputenc}
\usepackage{amstext}
\usepackage{eqname}
\usepackage{amssymb}
\usepackage{graphicx}
\usepackage{array}
\makeatletter
 \usepackage[ruled,vlined]{algorithm2e}
\newtheorem{thm}{\protect\theoremname}
\newtheorem{lemma}{Lemma}

\usepackage{xcolor}
\newcommand{\NN}{\mathbb{N}}
\newcommand{\RR}{\mathbb{R}}

\makeatother

\usepackage{babel}
\addto\captionsamerican{\renewcommand{\theoremname}{Theorem}}
\addto\captionsenglish{\renewcommand{\theoremname}{Theorem}}
\providecommand{\theoremname}{Theorem}

\title{\LARGE \bf Provably Safe and Deadlock-Free Execution of Multi-Robot Plans\\ under Delaying Disturbances}

\author{Michal Cap$^{*,1,2}$ and Jean Gregoire$^{*,1}$ and Emilio Frazzoli$^{1}$ %
\thanks{$^{*}$Equal contribution}%
\thanks{$^{1}$Michal Cap, Jean Gregoire and Emilio Frazzoli are affiliated with LIDS, Massachusetts Institute of Technology.}%
\thanks{$^{2}$Michal Cap is also affiliated with Dept. of Computer Science, FEL, CTU in Prague.~{\tt\small michal.cap@fel.cvut.cz}}%
}

\begin{document}
\maketitle
\thispagestyle{empty}
\pagestyle{empty}

\begin{abstract}
One of the standing challenges in multi-robot systems is the ability
to reliably coordinate  motions of multiple robots in environments
where the robots are subject to disturbances. We consider disturbances that force the robot to temporarily stop and delay its advancement along its planned trajectory which can be used to model, e.g., passing-by humans for whom the robots have to yield. Although reactive
collision-avoidance methods are often used in this context, they
may lead to deadlocks between robots. We design a multi-robot control strategy for
executing coordinated trajectories computed by a multi-robot trajectory planner and give a proof that the strategy is safe and deadlock-free
even when robots are subject to delaying disturbances. Our simulations show that the proposed strategy scales significantly better with the intensity of disturbances than the naive liveness-preserving approach. The empirical results further confirm that the proposed approach is more reliable and also more efficient than state-of-the-art reactive techniques.
\end{abstract}

\section{Introduction}

The advancements of robotics during the last decade lead to the emergence
of production-scale multi-robot systems consisting of a large number
of cooperating mobile robots. One of the prime examples is the warehouse
management systems developed by Kiva Systems (now Amazon Robotics)
that employs mobile robots to fetch products for human warehouse
workers. To simplify the problem, the robots in such systems typically
operate in a dedicated human-excluded area. To further improve impact
of robotics in manufacturing, the current research is focused on the development
of robots that can safely share floor with human workers. One of the
challenges posed by mixed human-robot systems is how to coordinate the motions
of the robots given that their motion can be disturbed by humans.

The problem of coordinating motions of multiple robots can be approached
either from a control engineering perspective by employing the \emph{reactive paradigm} to collision avoidance or from an AI perspective by employing the \emph{deliberative paradigm} that amounts to planning coordinated trajectories for the robots.

In the \emph{reactive paradigm}, the robot follows the shortest path to its
current destination and attempts to resolve collision situations as
they appear, locally. Each robot periodically observes positions and velocities
of other robots in its neighborhood. If there is a potential future
collision, the robot attempts to avert the collision by adjusting
its immediate heading and velocity. A number of methods have been
proposed~\cite{vanDenBerg2008RVO,Guy2009_ClearPath,AlonsoMora:kb}
that prescribe how to compute such collision-avoiding velocity in
a reciprocal multi-robot setting, with the most prominent one being
ORCA~\cite{vanBerg2011_ORCA}. These approaches are widely used in
practice thanks to their computational efficiency -- a collision-avoiding
velocity for a robot can be computed in a fraction of a millisecond~\cite{vanBerg2011_ORCA}.
However, these approaches resolve collisions only locally and thus
they cannot guarantee that the resulting motion will be deadlock-free
and that all robots will always eventually reach their destination. 

In the \emph{deliberative paradigm}, the system first searches for a set of globally
coordinated collision-free trajectories from the origin position to
the destination of each robot. After the planning has finished, the
robots start following their respective trajectories. If the robots 
execute the resulting joint plan precisely (or within some predefined
tolerance), it is guaranteed that they will reach their destination
while avoiding collisions with other robots. It is however known that
the problem of finding coordinated trajectories for a number of mobile
objects from the given start configurations to the given goal configurations
is intractable. More precisely, the coordination of disks amidst polygonal
obstacles is NP-hard~\cite{SpirakisY84_Strong_NP_Hardness_of_Moving_Many_Discs}
and the coordination of rectangles in a rectangular room is PSPACE-hard~\cite{hopcroft84}.
Even though the problem is relatively straightforward to formulate
as a planning problem in the Cartesian product of the configuration
spaces of the individual robots, the solutions can be very difficult
to find using standard heuristic search techniques as the joint state-space
grows exponentially with the number of robots. That is why heuristic techniques
such as prioritized planning~\cite{Erdmann87onmultiple,cap_2015_b}
are often used in practice. In fact, for multi-robot coordination
in so-called well-formed infrastructures, a solution is known to always
exist and it can be found in polynomial time using decoupled planning
techniques~\cite{cap_2015_a,cap_2015_b}.

The main drawback of the {deliberative paradigm} is that robots are
guaranteed to reach their goals only if they follow the planned trajectories
\emph{precisely} both in space and time. Although the spatial component
of the trajectory can be typically followed with reasonable tracking
error using existing path tracking techniques, precise tracking of
the temporal component is hard to achieve in environments that involve
humans due to the typical requirement that the robots must robots yield to people, who are however hard to predict.
\begin{figure}[h]
\centering
\includegraphics[width=6.5cm]{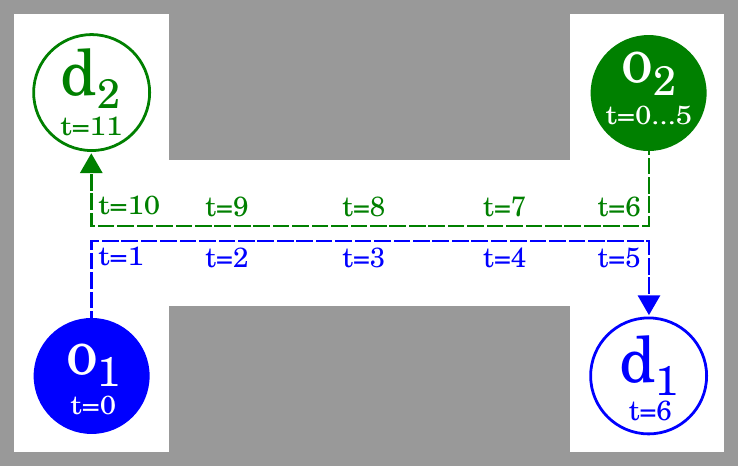} 
\caption{Example coordination problem. 
First robot desires to move from origin $o_1$ to destination $d_1$, the second robot desires to move from $o_2$ to $d_2$. A possible solution found by a multi-robot trajectory planner is indicated by the dashed line, where the planned position for each timepoint is  annotated as $t=\ldots $ shown next to the trajectory. The plan dictates that the robot~1 goes through the corridor first. Robot~2 waits at its origin until robot 1 leaves the corridor and then proceeds through the corridor towards its destination. Now suppose that robot~1 is subject to a disturbance and enters the corridor with the delay of 5 time units. If robot~2 does not reflect on the delay and continues moving according to the original plan, both robots engage in a heads-on collision or deadlock, depending on whether some low-level safety mechanism prevents the robots from physically crashing, in the center of the corridor.} 
\label{fig:example-coordination-problem}
\end{figure}
Observe that if one of the robots is delayed, and the others simply
continue advancing along their trajectory, the robots may end up in
a collision or deadlock (see Figure~\ref{fig:example-coordination-problem} for an example). 
One possible approach to deal with such a disturbance is to find new
coordinated trajectories from the current positions of all the robots
to their destinations. 
Although replanning may sometimes work, it is
usually too computationally intensive to be performed in a feedback loop
and there is no guarantee that the resulting instance will be solvable
in practical time. Another strategy, that we will refer to as ALLSTOP, is to stop all robots in the team every time any single robot is forced to stop. Although this strategy preserves liveness of the coordinated plan, it is clearly inefficient. 

In this work, we show that stopping the entire team when a disturbance stops a single robot is in most situations unnecessary  and provide a simple control rule specifying when a robot can proceed along its trajectory without risking future deadlock or collisions. Therefore, the proposed method can be used to generalize guarantees of existing multi-robot trajectory planning techniques to environments, where robots might be forced to temporarily stop. 
Our experiments suggests that in such environments, executing preplanned trajectories using the proposed control rule is more reliable and more efficient than coordinating robots using reactive collision-avoidance techniques.

The proposed technique may also act as an enabler allowing generalization of hybrid architectures, which proved a powerful paradigm in the context of intersection coordination~\cite{kowshik2011provable}, towards more general multi-robot coordination problems. 

The paper is structured as follows. Section~\ref{sec:problem} formulates the multi-robot plan execution problem under non-deterministic disturbances. Section~\ref{sec:control-scheme} reformulates the problem in the coordination space and designs a control strategy. Section~\ref{sec:analysis} proves that the proposed approach is safe and deadlock-free. Finally, Section~\ref{sec:evaluation} evaluates our control scheme through simulations and Section~\ref{sec:conclusion} concludes the paper.

\section{Problem Formulation}
\label{sec:problem}

Consider a 2-d environment $\mathcal{W}\subseteq\mathbb{R}^{2}$ populated
by $n$ identical disc-shaped holonomic robots indexed $1,\ldots,n$.
Their body has radius $r$ and they can travel at unit maximum speed.
Let $\pi_{1},\ldots,\pi_{n}$ be feasible collision-free trajectories
from the desired origin positions to the desired destination positions obtained from a multi-robot trajectory planner.
The discrete-time trajectory of robot $i$ is a function $\pi_{i}(t):\:\{0, \ldots , T\}\rightarrow\mathcal{W}$, where $T$ denotes the time step when the last robot reaches its destination.
The state of the system is described in terms of position $x_{i}\in\{0, \ldots , T\}$
of each robot $i$ along its trajectory $\pi_{i}$, i.e. if a robot
is in state $x_{i}$, the robot is at spatial position $\pi_{i}(x_{i})$.
The control variables are the advancement decisions for each robot,
where the decision whether robot $i$ should continue along its trajectory
or stop at timestep $t$ is denoted as $a_{i}(t)\in\{0,1\}$. The advancement of each
robot is subject to an exogenous multiplicative disturbance $\delta_{i}(t)\in\{0,1\}$,
where $\delta_{i}(t)=0$ models the situation when the robot is forced
to stop. Then, the discrete-time system dynamics is governed by the equation:
\[
\forall t\in\NN,~ x_{i}(t+1)=x_i(t)+a_{i}(t)\cdot\delta_{i}(t).
\]
The state $x_i(t)$ of robot $i$ can be intuitively interpreted as a position in the plan $\pi_i$ at time $t$. Note that this position is measured in time units. The model of robot dynamic we use is further illustrated in Figure~\ref{fig:model}.

\begin{figure}[ht]
\begin{center}
\includegraphics[width=0.7\linewidth]{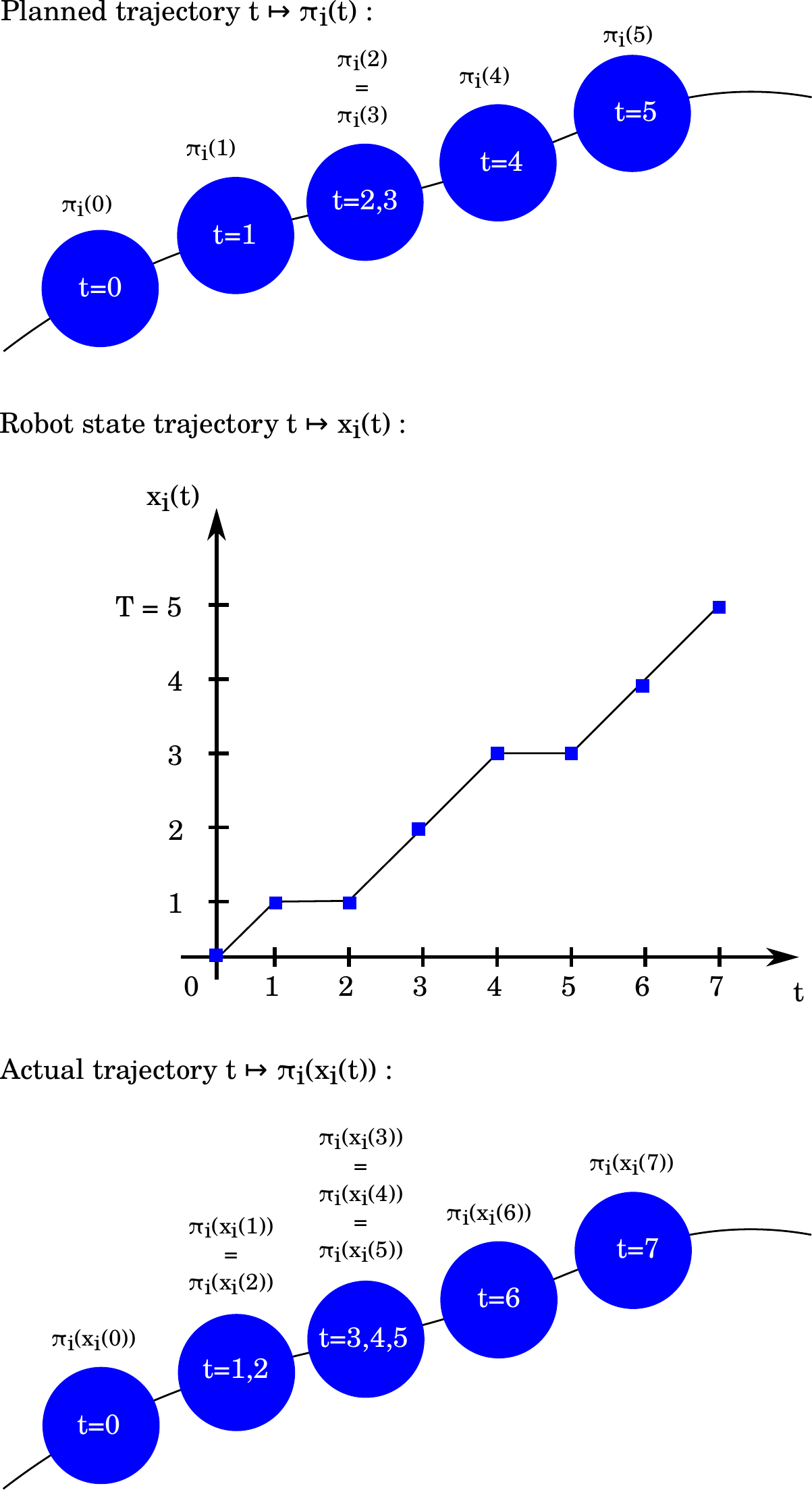}
\end{center}
\caption{\textbf{Top:} An example output of the planner for robot $i$. The balls indicate the time evolution of the planned trajectory $t \mapsto \pi_i(t)$ for robot $i$. The curve represents the geometric path in $\RR^2$ followed by robot $i$. 
\textbf{Middle:} State trajectory $t \mapsto x_i(t)\in[0,5]$ encodes the time evolution of position of robot $i$ along its planned trajectory. We can see that during time intervals $[1,2]$ and $[4,5]$ the robot makes no progress in its plan. Also, the fact that the state trajectory is below the diagonal can be interpreted as the robot lagging behind its plan. 
\textbf{Bottom:} The actual trajectory of the robot is $t \mapsto \pi_i(x_i(t))$. One can see that robot $i$ will follow the same geometric path in $\RR^2$ as planned. However, the time evolution along this geometric path is different. In particular, robot $i$ will reach its goal at time $7$ instead of $5$.
}
\label{fig:model} 
\end{figure}

Our objective is to design a multi-robot controller $G(x_{1},\ldots,x_{n})$
that takes the current position of each robot and returns the advancement
decision for the robots $(a_{1},\ldots,a_{n})$ such that from the
initial state $\mathbf{x}_{0}=(0,\ldots,0)$ at time $t=0$, the system
is 
\begin{itemize}
\item collision-free, i.e., the robots should never collide 
\[
\forall t\:\forall i,j\ i \neq j:\quad\left|\pi_{i}(x_{i}(t))-\pi_{j}(x_{j}(t))\right|>2r,
\]

\item deadlock-free, i.e. the robots should eventually reach their destination
\[
\forall i\:\exists t_{f}\:\forall t\geq t_{f}:\quad x_{i}(t)=T \text{, and}
\]

\item efficient, i.e. disturbance affecting one robot should not lead to
stopping the entire system.
\end{itemize}
In the next section we present a control scheme that satisfies the
properties stated above. We remark that the presented scheme combines
advantages of planning and reactive approach to multi-robot coordination
in that it guarantees liveness and in the same time retains freedom
of action allowing robots to deviate from the planned trajectory and
handle the disturbance.

\section{Control Scheme}
\label{sec:control-scheme}
In this section we will introduce a simple control rule,  which we will refer to as Robust Multi-Robot Trajectory Tracking Strategy (RMTRACK), that preserves safety and liveness of given multi-robot trajectory even under delay-inducing disturbances. To be able to concisely represent conditions about safe and unsafe mutual configurations of robots, we will make use of now standard coordination space formalism~\cite{ODonnell1989,LaValle06}. Our approach relies on transformation of the $n$-robot coordination problem in the physical space into the problem of finding a collision-free path in an $n$-dimensional abstract space called coordination space. 

\subsection{Coordination space}

A coordination space $X$ for $n$ robots is defined as an $n$-dimensional
cube $X=\{0, \ldots , T\}^{n}$. A point in a coordination space encodes the state of all robots, i.e. the position along their trajectories. Let $(x_{i},x_{j})$ be a point in a coordination space of two robots $i$ and $j$ with trajectories $\pi_{i}$
and $\pi_{j}$. Then, the point $(x_{i},x_{j})$ is said to be in
collision denoted as $c_{ij}(x_{i},x_{j})$ if $\left|\pi_{i}(x_{i})-\pi_{j}(x_{j})\right|<2r$.
The set of all states in coordination space of robots $i$ and $j$ representing collision between the two robots is denoted as $C_{ij}$ and defined as
\[
C_{ij}:=\left\{ (x_{i},x_{j})\ |\ c_{ij}(x_{i},x_{j})\right\}. 
\]
Analogically, the collision region $C$ in the coordination space of $n$ robots is defined as 
$$
\begin{array}{rl}
C&:=\left\{ (x_{1},\ldots,x_{n})\ |\ \exists i,j\ i\neq j:\ c_{ij}(x_{i},x_{j})\right\} \\
&\,=\left\{ (x{}_{1},\ldots,x_{n})\ |\ \exists i,j\neq i:\ (x_{i},x_{j})\in C_{ij}\right\}. \\
\end{array}
$$

By translating the original problem to the coordination space, our problem is now to design a
controller that ensure that the trajectory of the multi-robot system in the coordination space $t \in \{0,\ldots,T\} \mapsto (x_{1}(t),\ldots,x_{n}(t))$ starting from $(0,\cdots,0)$ will reach $(T,\ldots,T)$ in finite time $t_f$ and at all times remains in collision-free region $X \setminus C$ of the coordination space $X$.

Observe that in the absence of disturbances the trajectory of the system in the coordination space will be a "diagonal" line segment connecting points $(0,\ldots,0)$ and $(T,\ldots,T)$. Also, it is important to notice that the "diagonal" in the coordination space $t\in\{0,T\} \mapsto (t, \ldots , t)\in\{0,T\}^{n}$ is necessarily collision-free with respect to $C$ since the planned trajectories ($\pi_{1},\ldots,\pi_{n}$) are assumed to be collision-free. Now, the problem is to design a
control law that ensures that collision region $C$ is avoided even in the presence of disturbances.

\subsection{Control Law} \label{sec:control-law}
We now introduce RMTRACK control law that ensures that the actual trajectory in coordination space under disturbances avoids the collision region $C$, while simultaneously always making progress towards point $(T,\ldots,T)$. The multi-robot control law $G(x_1,\ldots,x_n)$ is decomposed into collection of control laws  $\{ G_i(x_1,\ldots,x_n) \}$, each governing the advancement of robot $i$. The control law for single robot $i$ is defined as follows:
\begin{equation}
\begin{array}{l}
G_{i}(x_{1}, \ldots , x_{n}):= \\
\quad\left\{
\begin{array}{ll}
0\quad\text{ if } x_{i}=T\text{ or if \ensuremath{\big[}}\exists j:\ x_{i}>x_{j}\;\text{ and }\\
\quad \quad \quad C_{ij}\cap(\{x_{i}+1\}\times\{x_{j}, \ldots , x_{i}+1\})\neq\emptyset\big]\\
1\quad\text{ otherwise.}
\end{array}\right.
\end{array}
\label{eq:control-law}
\end{equation}
As we can see, the control law allows robot $i$ to proceed only if the line segment from $(x_i+1,x_j)$ to $(x_i+1,x_i+1)$ is collision-free in the coordination space $X_{ij}$ for every other robot $j$. 
The mechanism is illustrated in Figure~\ref{fig:illustration-control-law-condition-check}.  
The effect of application of such a control law is shown in example in Figure~\ref{fig:illustration-control-law}. 

\begin{figure}[ht]
\begin{centering}
\includegraphics[scale=0.36]{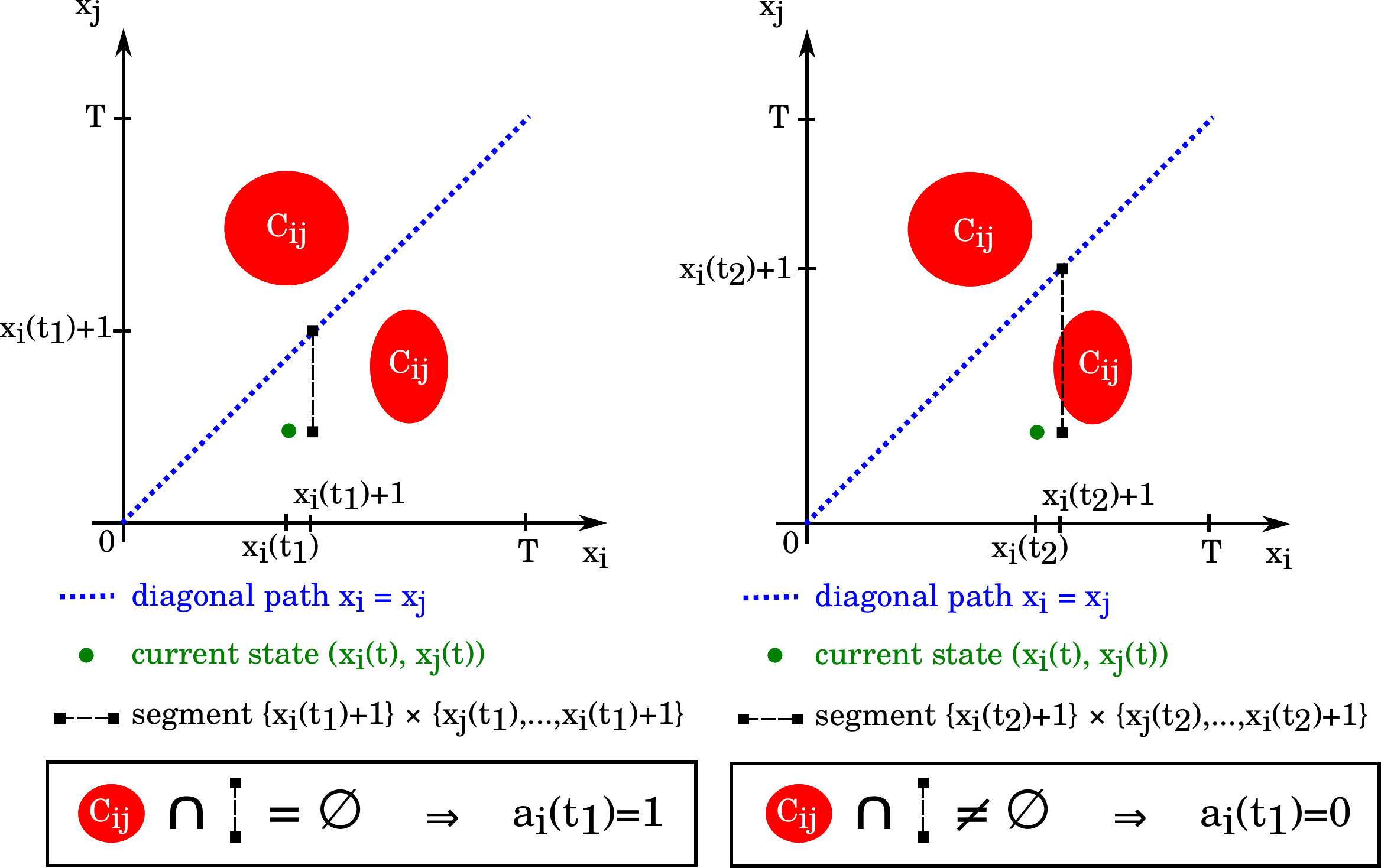}
\end{centering}
\caption{Illustration of the control law computation for robot $i$ with respect to single other robot $j$. \textbf{Left:} The segment $\{x_{i}+1\}\times\{x_{j}, \ldots , x_{i}+1\}$ is  collision-free with $C_{ij}$ at time $t_1$, therefore the robot $i$ is commanded to proceed, i.e. $a_i(t_1)=1$. \textbf{Right:} The segment is not collision-free with $C_{ij}$ at time $t_2$, therefore the robot $i$ is commanded to stop, i.e. $a_i(t_2)=0$.}
\label{fig:illustration-control-law-condition-check} 
\end{figure}

\begin{figure}[ht]
\begin{center}
\includegraphics[scale=0.45]{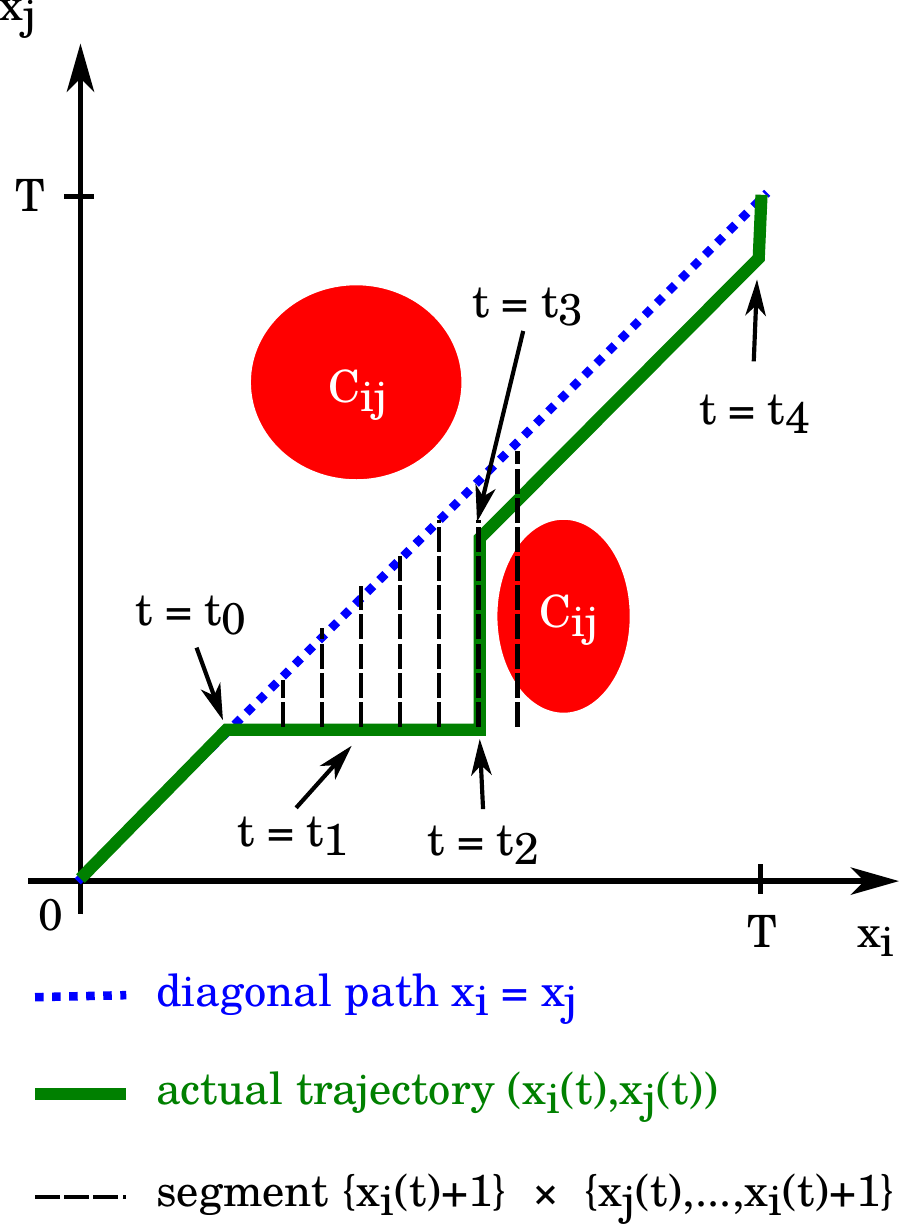}
\end{center}
\caption{Trajectory of the robots in the coordination space under control law $G$. At time $t=t_0$, a disturbance makes robot $j$ stop ($\delta_j(t_0)=0$). The system leaves the diagonal path, and the control law allows robot $i$ to proceed until time $t=t_2$ where robot $i$ is commanded to stop, i.e., $a_i(t_2)=0$ (see Figure~\ref{fig:illustration-control-law-condition-check}). Finally, disturbance for robot $j$ goes away, and robot $j$ proceeds. It's only at time $t=t_3$ that the control law allows robot $i$ to proceed.  Robot $i$ reaches its goal at time $t=t_4$ and robot $j$ right after. No collision occurred and note how the control law ensured that the actual trajectory in the coordination space remains on the same side as the diagonal path with respect to each connected component of the obstacle region. This is what enables to guarantee liveness, as changing the homotopy class  may lead to a deadlock configuration with respect to other robots.} 

\label{fig:illustration-control-law} 
\end{figure}

The control law allows robots to deviate from the planned trajectory in the coordination space, 
but ensures that the actual trajectory the robots follow is homotopic to the planned trajectory, 
i.e. to the "diagonal" path. 
This is ensured by requiring that the trajectory in coordination space of any two different robots $i$ and $j$ stay at the same side of each connected component of $C_{ij}$ as planned. Translated back to the original multi-robot formulation in $\mathbb{R}^2$ workspace, the above controller ensures that the robots will traverse overlapping parts of their geometric paths in the same order as implicitly specified in the planned trajectories.
We note that the usefulness of navigation under homotopic constraints has been previously noticed both in the context of single robot control~\cite{brock2002elastic} and multi-robot coordination~\cite{gregoire2013robust}.

\section{Theoretical Analysis}
\label{sec:analysis}
In this section we show that under certain technical assumptions, the
RMTRACK control law satisfies collision-freeness and liveness properties, i.e. the robots are guaranteed not to collide and eventually reach their goal positions.

\subsection{Collision-freeness}

In order to show collision-freeness we make the following technical assumption. We assume that there is a $1$-margin between the diagonal path and the obstacle region in the coordination space, i.e.,
\begin{equation}
\forall t\in\{0, \ldots ,T-1\},~\forall i\neq j,~\left\{
\begin{array}{ll}
(t+1,t)\notin C_{ij}\\
\text{and}\\
(t,t+1)\notin C_{ij}
\end{array}
\right.\label{eq:1margin}
\end{equation}

This assumption is as typically negligible as the geometric distance traveled by a robot in one time step is small. However, it can be satisfied by planning the reference trajectories with robots with slightly larger bodies. 

\begin{lemma}
Under control law $G$ and assuming that Equation~\ref{eq:1margin} holds, we have for all $t\in\NN$ and for all $i,j\in\{1, \ldots ,n\}$
s.t. $x_{i}(t)\geq x_{j}(t)$ : 
\begin{equation}
C_{ij} \cap (\{x_{i}(t)\} \times \{x_{j}(t) , \ldots , x_{i}(t)\}) = \emptyset\eqname{$E_{i,j,t}$}
\end{equation}
\label{lemma:collision-free}
\end{lemma}

\begin{proof}
Initially, we have $x_{1}(0)=x_{2}(0)= \ldots =x_{n}(0)=0$ and the
state of robots does not belong to $C_{ij}$, so that ($E_{i,j,0}$)
holds for all $i,j\in\{1, \ldots , n\}$.

Now, assume that ($E_{i,j,t}$) holds at some arbitrary time step $t\in\NN$ for all $i,j\in\{1, \ldots , n\}$ s.t. $x_{i}(t)\geq x_{j}(t)$. 

For each $i,j\in\{1, \ldots , n\}$, consider two options:
\begin{itemize}
\item If $x_i(t)=x_j(t)$, then $(E_{i,j,t})\equiv (E_{j,i,t})$ hold and consider three options:
\begin{itemize}
\item If $x_i(t+1)=x_j(t+1)$, $(E_{i,j,t+1})\equiv (E_{j,i,t+1})$ will be satisfied as planned trajectories are collision-free.
\item If robot $i$ moves one step forward while the other one $j$ is stopped, then by Equation~\ref{eq:1margin}, we have $(x_i(t+1),x_j(t+1))=(x_i(t)+1,x_i(t))\notin C_{ij}$, and as planned trajectories are collision-free $(x_i(t+1),x_i(t+1))\notin C_{ij}$, so that $\{x_i(t+1)\}\times\{x_j(t+1),x_i(t+1)\} \cap C_{ij}\neq 0$, so that ($E_{i,j,t+1}$) is satisfied with $x_i(t+1)>x_j(t+1)$.
\item If robot $i$ moves one step forward while the other one $j$ is stopped,, we use the symmetric reasoning to obtain that ($E_{j,i,t+1}$) holds.
\end{itemize}
\item If $x_i(t)>x_j(t)$ then we have $x_i(t+1)\geq x_j(t+1)$ and consider three options:
\begin{itemize}
\item If neither of robots moves, ($E_{i,j,t+1}$) will still be obviously satisfied as $(E_{i,j,t+1})\equiv (E_{i,j,t})$ which holds.
\item If robot $i$ does not move, then we have:
\begin{eqnarray*}
\{x_{i}(t+1)\}\times\{x_{j}(t+1), \ldots , x_{i}(t+1)\}\\
=\{x_{i}(t)\}\times\{x_{j}(t+1), \ldots , x_{i}(t)\}\\
\subseteq\{x_{i}(t)\}\times\{x_{j}(t), \ldots , x_{i}(t)\}
\end{eqnarray*}
which does not intersect $C_{ij}$ as ($E_{i,j,t}$) holds, so that ($E_{i,j,t+1}$) is satisfied.
\item If robot $i$ moves, then by construction of $G$ and because $x_i(t)>x_j(t)$, we have:
\begin{eqnarray}
C_{ij}\cap(\{x_{i}(t)+1\}\times\{x_{j}(t), \ldots , x_{i}(t)+1\}) = \emptyset \label{eq:lemma-proof}
\end{eqnarray}
Taking into account that $x_i(t+1)=x_i(t)+1$ and $x_j(t+1)\in\{x_j(t),x_j(t)+1\}$, we obtain:
\begin{eqnarray*}
\{x_{i}(t+1)\}\times\{x_{j}(t+1), \ldots , x_{i}(t+1)\}\\
= \{x_{i}(t)+1\}\times\{x_{j}(t+1), \ldots , x_{i}(t)+1\}\\
\subseteq \{x_{i}(t)+1\}\times\{x_{j}(t), \ldots , x_{i}(t)+1\}
\end{eqnarray*}
which does not intersect $C_{ij}$ by Equation~\ref{eq:lemma-proof}, so that ($E_{i,j,t+1}$) holds.
\end{itemize}
\end{itemize}
By induction, we conclude that ($E_{i,j,t}$) is satisfied for all $t\in\NN$ and $i,j\in\{1, \ldots , n\}$ s.t. $x_{i}(t)\geq x_{j}(t)$.
\end{proof}
\begin{thm}
Under control law $G$ and and assuming that Equation~\ref{eq:1margin} holds, the trajectory in the coordination space is
collision-free, i.e. 
$$
\forall t\in\NN,~(x_{1}(t), \ldots , x_{n}(t))\notin C.
$$
\end{thm}
\begin{proof}
Take an arbitrary time step $t\in\NN$. Assume that $(x_{1}(t), \ldots , x_{n}(t))\in C$. Then, there exists $i,j$ such that $(x_i(t),x_j(t))\in C_{ij}$ and we can assume without loss of generality that $x_i(t)\geq x_j(t)$. This is in contradiction with Lemma~\ref{lemma:collision-free}.
\end{proof}

\subsection{Liveness}

First, we characterize our assumptions on disturbances. Clearly,
it is possible to construct a disturbance function that will prevent
the system from reaching configuration $(T, \ldots , T)$ under any control
law. For example, if for a given robot $i\in\{1, \ldots , n\}$, we have
$\forall t\in\NN,\:\delta_{i}(t)=0$, then it is impossible for robot
$i$ to reach its goal. Therefore, in the following analysis, we assume
that disturbances do not prohibit any of the robots from reaching
its goal, i.e. we consider systems in which disturbances may delay
any given robot for arbitrarily long, but the robot will eventually
be able to reach the goal. Formally, we say that disturbances are
\emph{non-prohibitive} if for any controller that satisfies
$$
\forall t\in\NN,\left\{ \begin{array}{ll}
x_{1}(t)=x_{2}(t)= \ldots =x_{n}(t)=T\\
\text{or}\\
\exists i\in\{1, \ldots , n\}:~x_{i}(t)<T\text{ and }a_{i}(t)=1
\end{array}\right.\text{,}
$$
the system will eventually reach the goal, i.e there exists $t_{f}\in\NN$
such that $(x_{1}(t_{f}), \ldots , x_{n}(t_{f}))=(T, \ldots , T)$. In other
words, as long as the controller lets at least one unfinished robot proceed at
any point of time, all robots will eventually reach their goal. The
time of goal achievement $t_{f}$ might be affected by disturbances,
but disturbances will not prevent goal achievement in finite time.

Under \emph{non-prohibitive} disturbances, the control law $G$
guarantees liveness:

\begin{lemma}
Under control law $G$, there is at least one robot proceeding at
any point of time, i.e., 
$$
\forall t\in\NN,~\left\{ \begin{array}{ll}
x_{1}(t)=x_{2}(t)= \ldots =x_{n}(t)=T\\
\text{or}\\
\exists i\in\{1, \ldots , n\}:~x_{i}(t)<T\text{ and }a_{i}(t)=1
\end{array}\right.
$$
\end{lemma}
\begin{proof}
Define $I(t)\subseteq \{1, \ldots , n\}$ as follows: 
$$
I(t):=\arg\min_{i} \; x_{i}(t)= \bigl\{ i: \forall  j \neq i \quad x_{j}(t) \geq x_{i}(t)  \bigr\} 
$$
$I(t)$ exists and is a non-empty set as $\{1, \ldots , n\}$ is finite.
By construction of $G$, we have: 
$$
\forall i\in I(t) \quad a_{i}(t)=G_{i}(x_{1}(t), \ldots , x_{n}(t))=1\text{ or }x_{i}(t)=T
$$
There are two scenarios: 
a) Either for all $i\in I(t)$, $x_{i}(t)=T$, then $\min_{i}x_{i}(t)=T$,
so that we have $x_{1}(t)=x_{2}(t)= \ldots =x_{n}(t)=T$. 
b) Or there exists some $i\in I(t)$ such that $x_{i}(t)<T$ and $a_{i}(t)=G_{i}(x_{1}(t) x_{n}(t))=1$. 
This concludes the proof. \end{proof}
\begin{thm}
Control law $G$ ensures liveness under non-prohibitive disturbances, i.e.
$$
\exists t_{f}:x_{1}(t)=x_{2}(t)=\ldots=x_{n}(t)=T
$$
\end{thm}
\begin{proof}
This is a direct consequence of the preceding lemma and of the assumption
made on disturbances.
\end{proof}

\section{Experimental Evaluation}
\label{sec:evaluation}

\begin{figure*}[t]
\begin{center}
    \begin{tabular}{>{\centering\arraybackslash} m{6cm} >{\centering\arraybackslash} m{6cm} >{\centering\arraybackslash} m{6cm}}
    Empty hall with 25 robots & Office corridor with 25 robots & Warehouse with 30 robots  \\  
    \includegraphics[width=6cm]{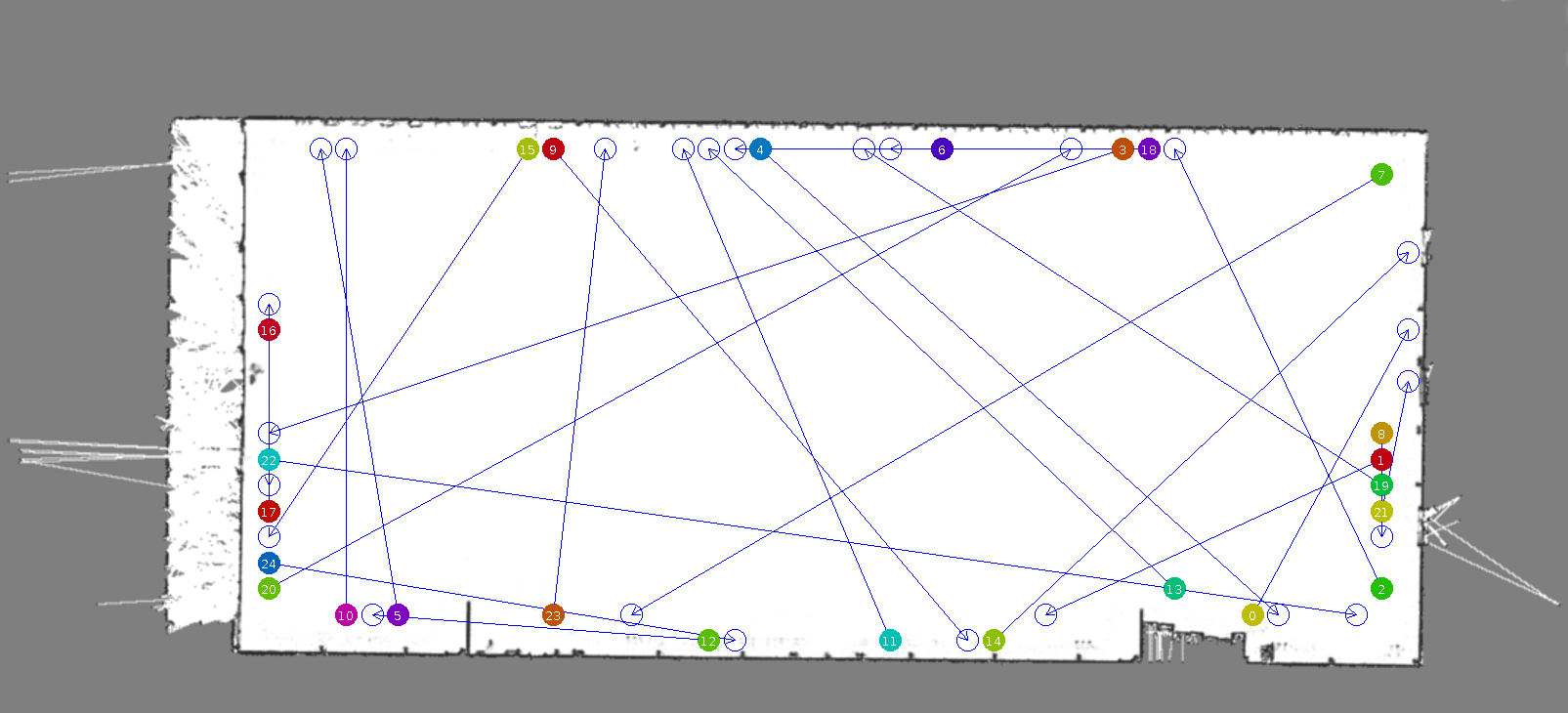} & 
    \includegraphics[width=6cm]{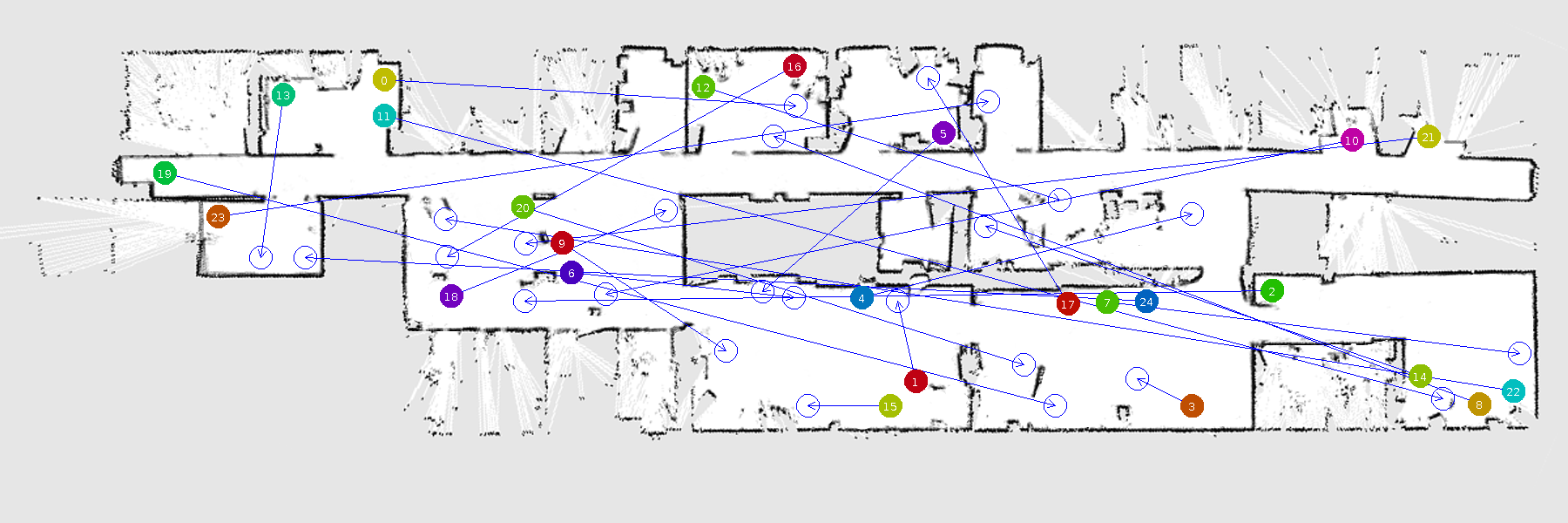} & 
    \includegraphics[width=4cm]{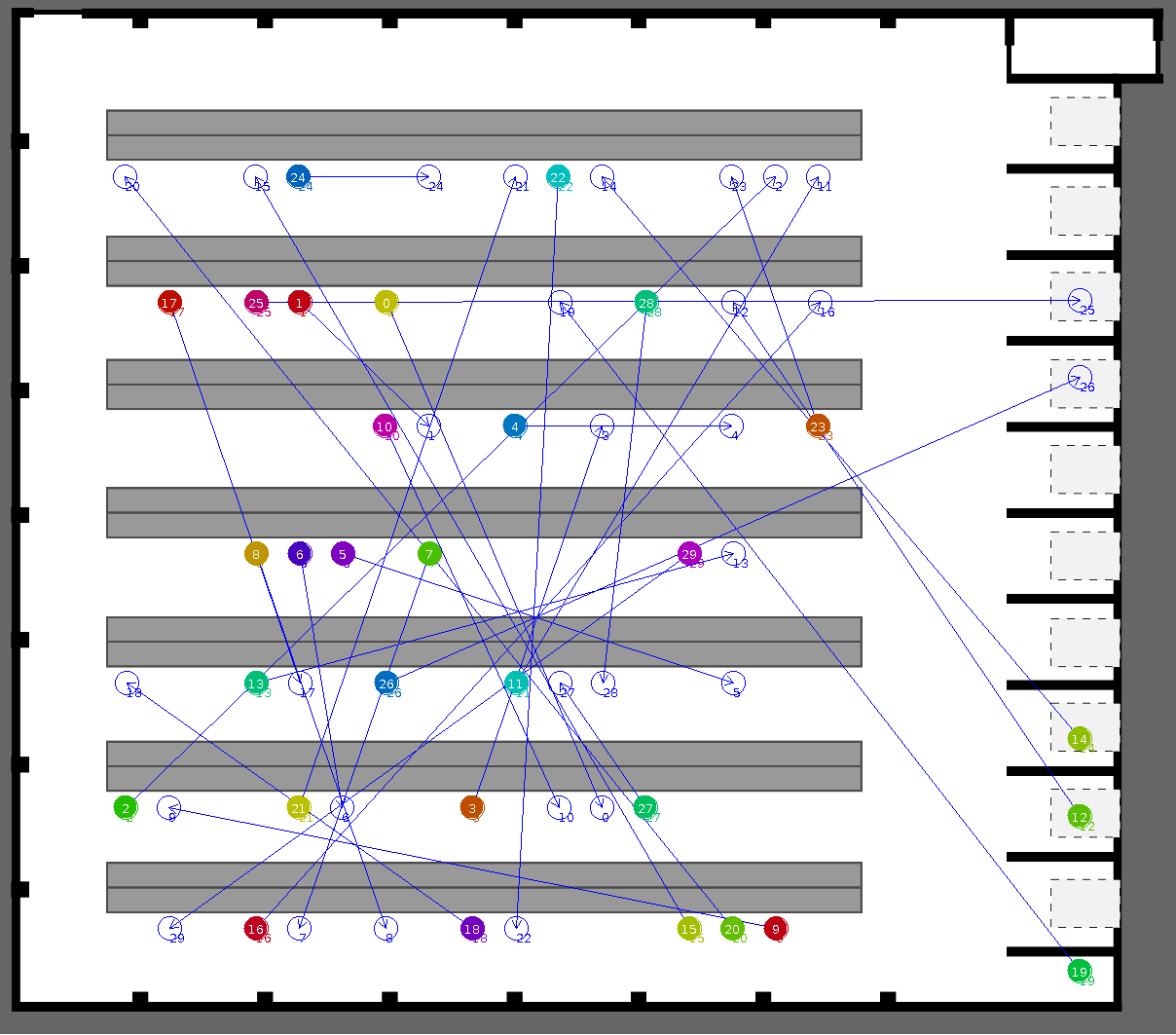} \\  
    \end{tabular} 
  \end{center}
\caption{Maps used for experimental comparison. The figures show an example problem instance in each environment. The filled circles represent robots. The arrows indicate the desired destination of each robot.} \label{fig:environments}
\end{figure*}

\begin{figure*}[t]
\begin{center}
    \begin{tabular}{c c c}
    Empty hall & Office corridor & Warehouse \\
    \includegraphics[width=2.7cm]{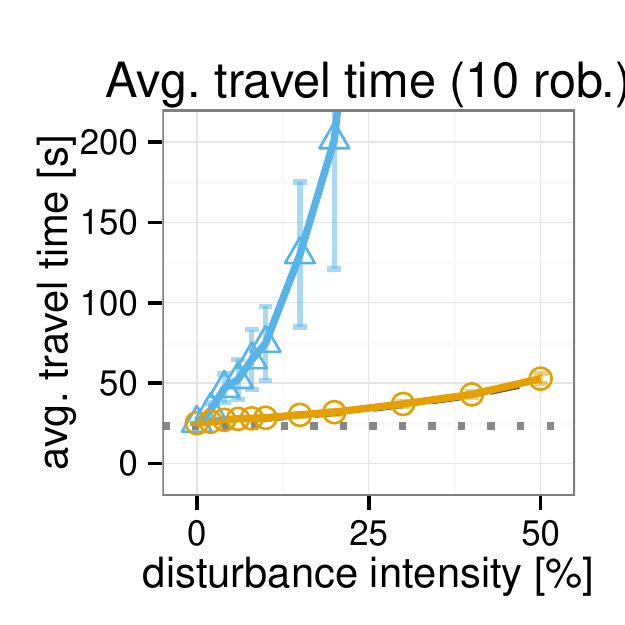}
    \includegraphics[width=2.7cm]{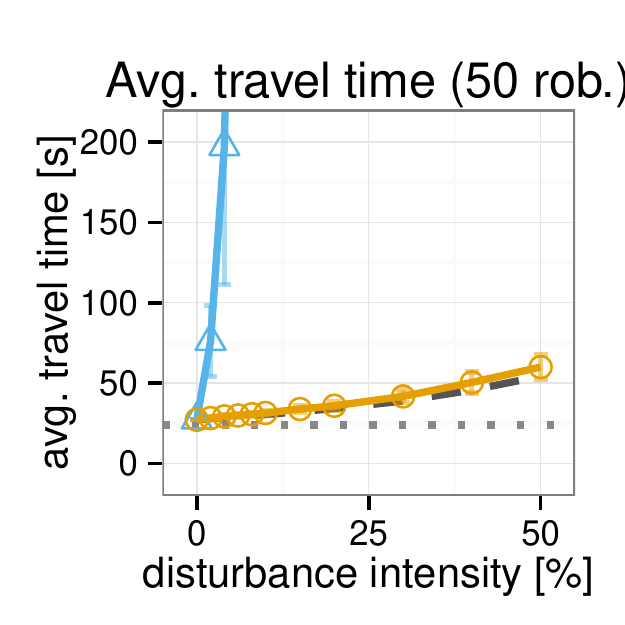} 
     & 
    \includegraphics[width=2.7cm]{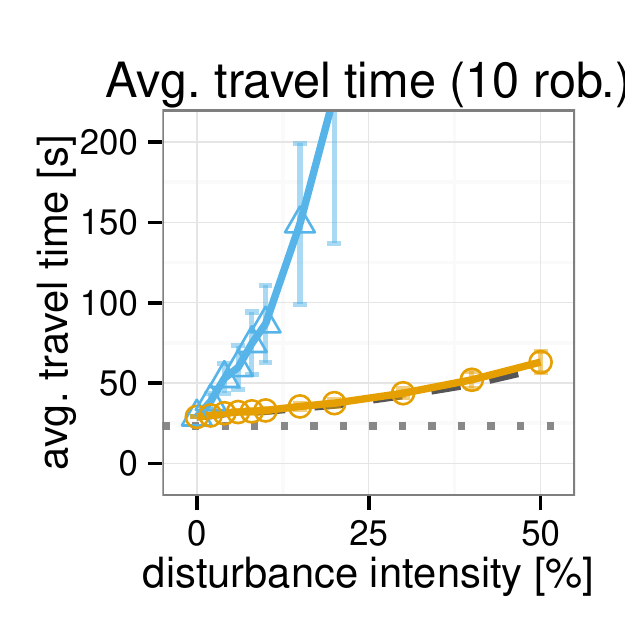}
    \includegraphics[width=2.7cm]{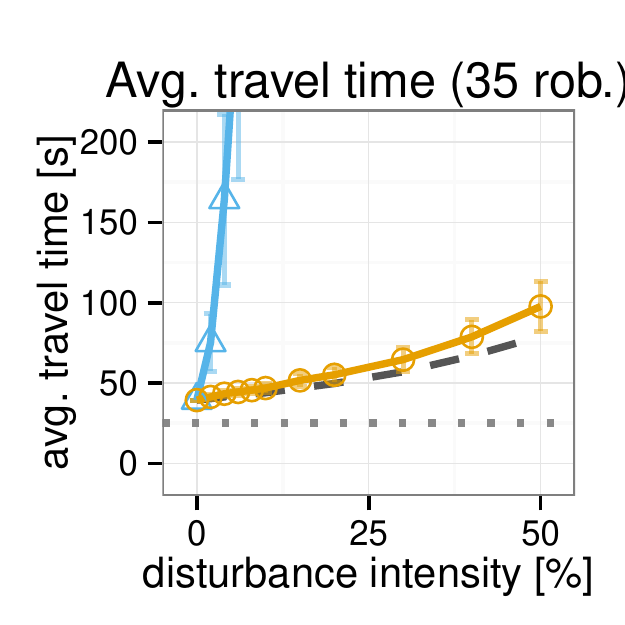} 
     & 
    \includegraphics[width=2.7cm]{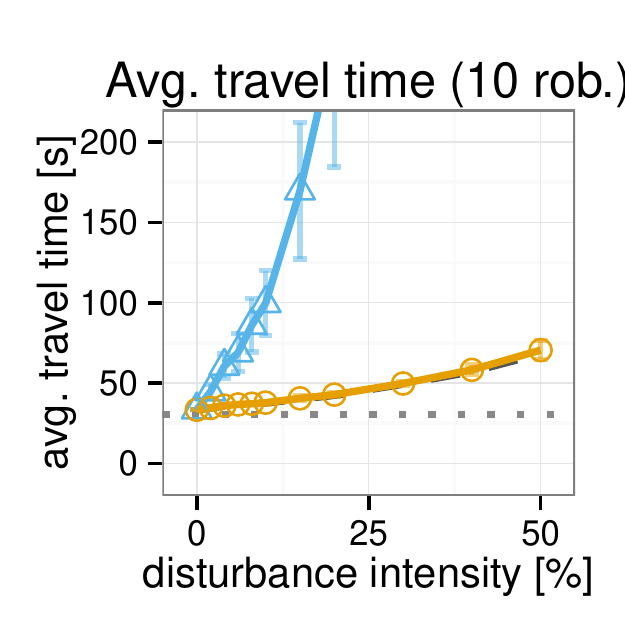}
    \includegraphics[width=2.7cm]{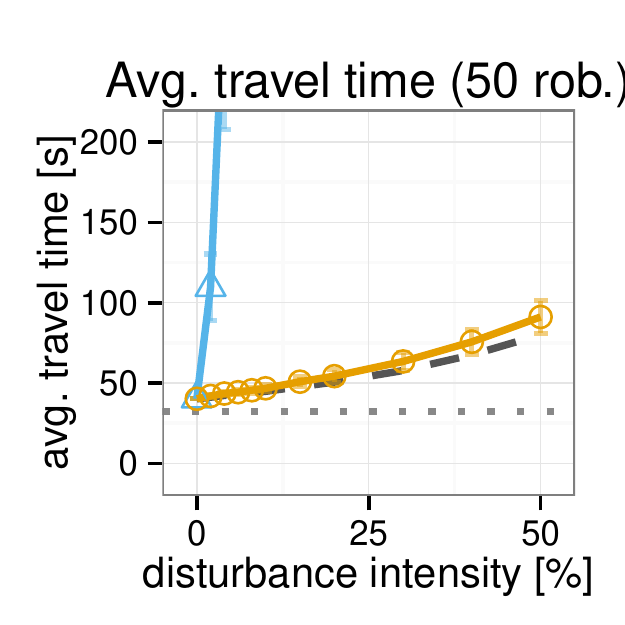}  
    \\  
    \end{tabular}
    \includegraphics[scale=0.5]{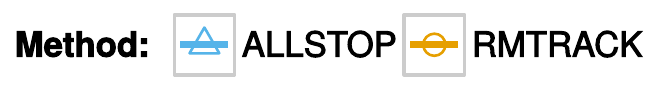} 

  \end{center}
\caption{Experimental comparison of ALL-STOP strategy with RMTRACK. Each datapoint is an average travel time of a single robot from its origin to its destination under the given disturbance intensity using one of the two evaluated control strategies. The dashed line represents the average lower bound on the travel time under the given disturbance intensity. The dotted line representes the average travel time from origin to destination assuming no disturbance and no need for coordination between robots. The bars represent standard deviation of the difference between the travel time under the evaluated algorithm and the lower bound travel time.} \label{fig:allstop-comparison}
\end{figure*}

\begin{figure*}[t]
\begin{center}
    \begin{tabular}{c c c}
    Empty hall & Office corridor & Warehouse \\
    \includegraphics[width=2.7cm]{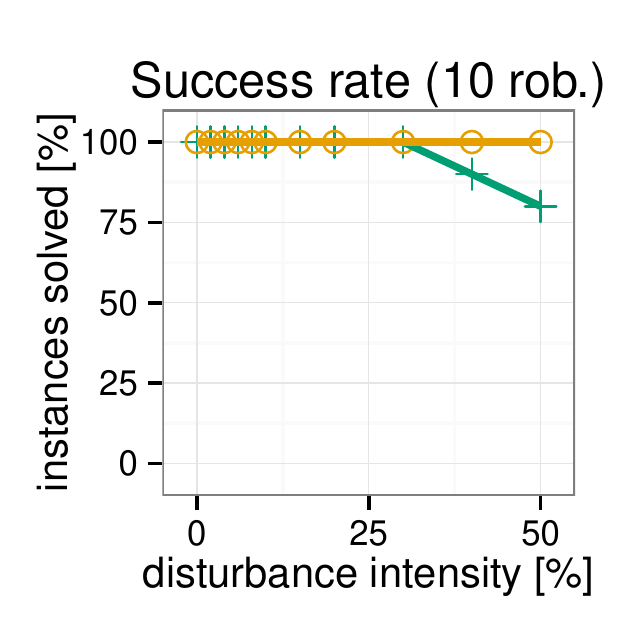}
    \includegraphics[width=2.7cm]{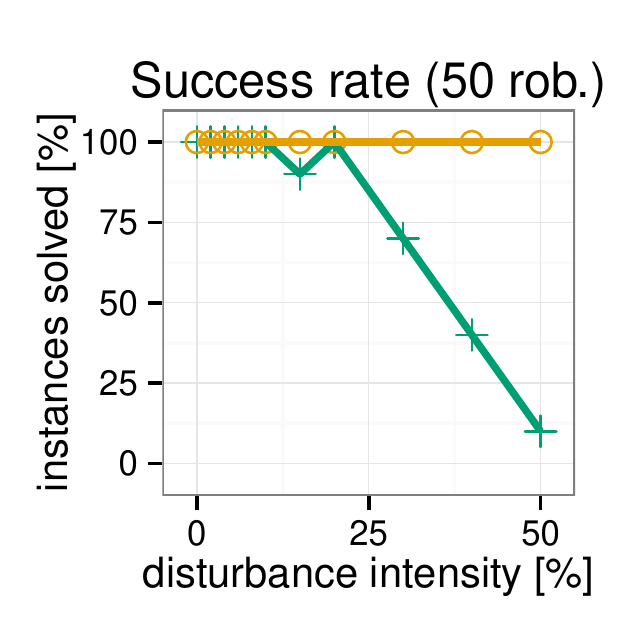}
     & 
    \includegraphics[width=2.7cm]{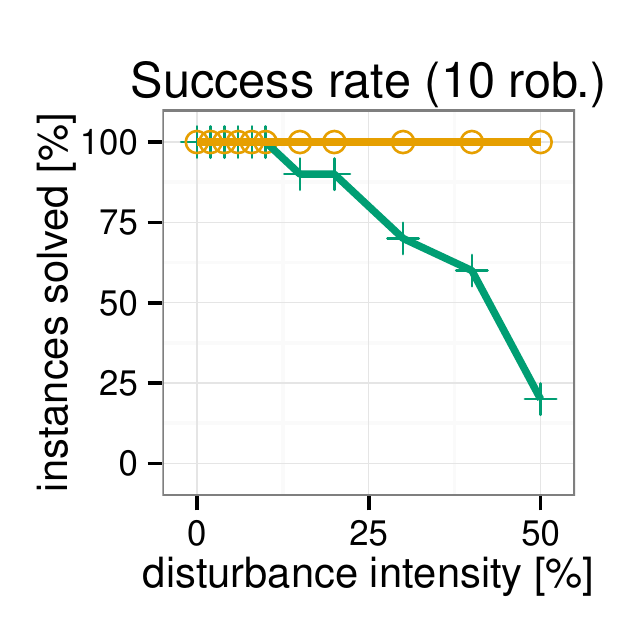}
    \includegraphics[width=2.7cm]{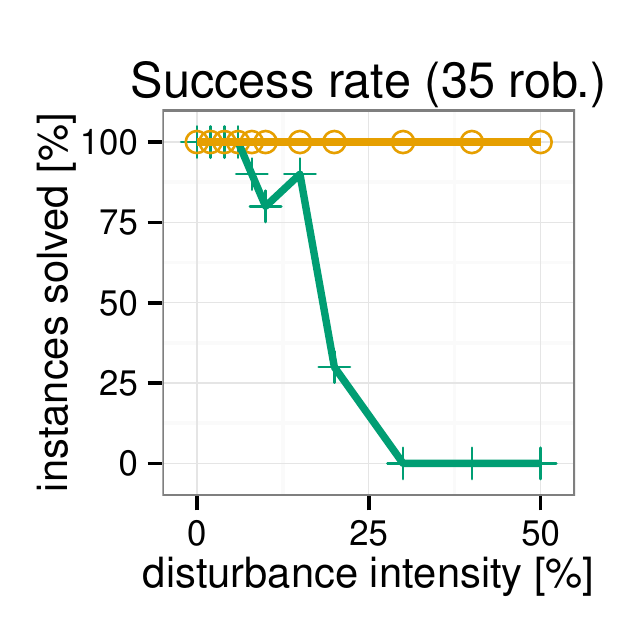}
     & 
    \includegraphics[width=2.7cm]{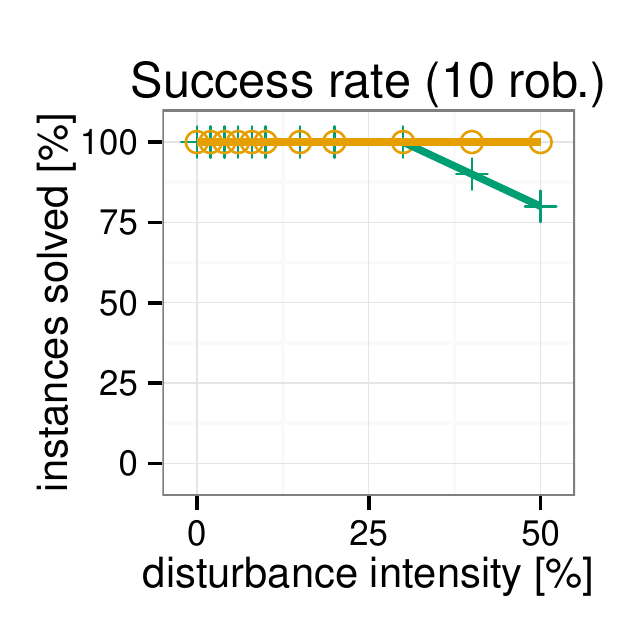}
    \includegraphics[width=2.7cm]{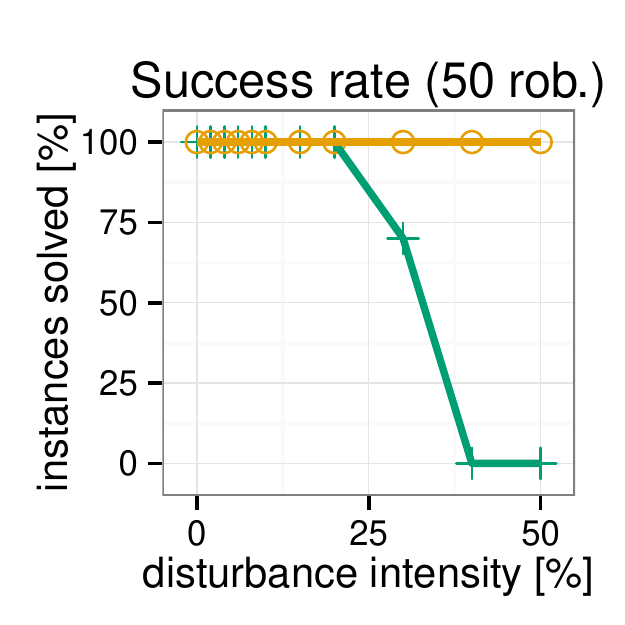}
    \\    
    \includegraphics[width=2.7cm]{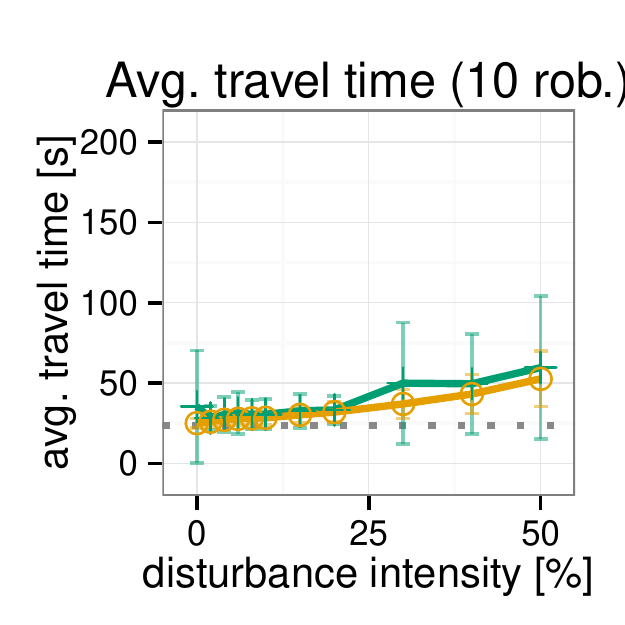}
    \includegraphics[width=2.7cm]{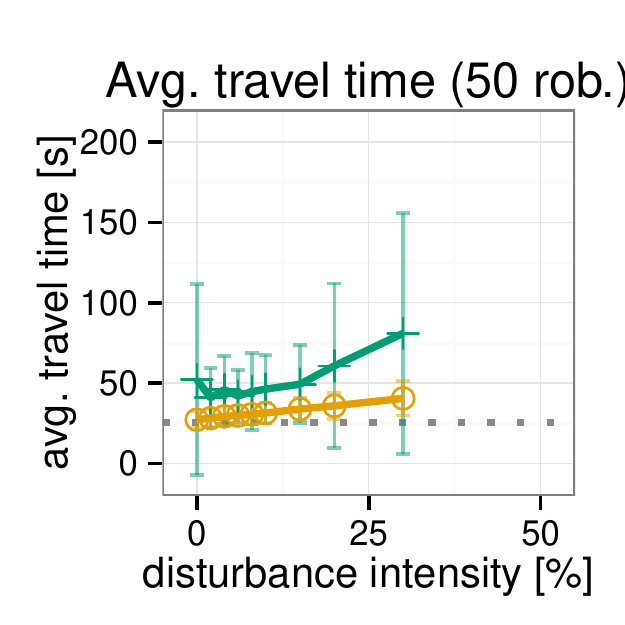} 
     & 
    \includegraphics[width=2.7cm]{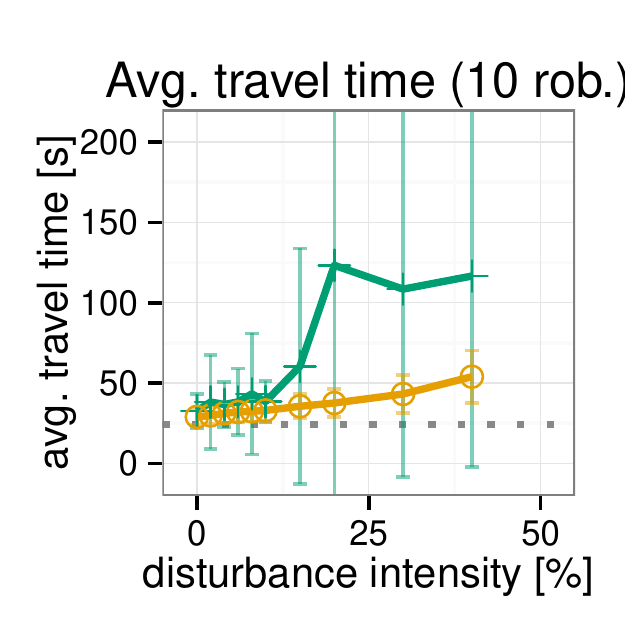} 
    \includegraphics[width=2.7cm]{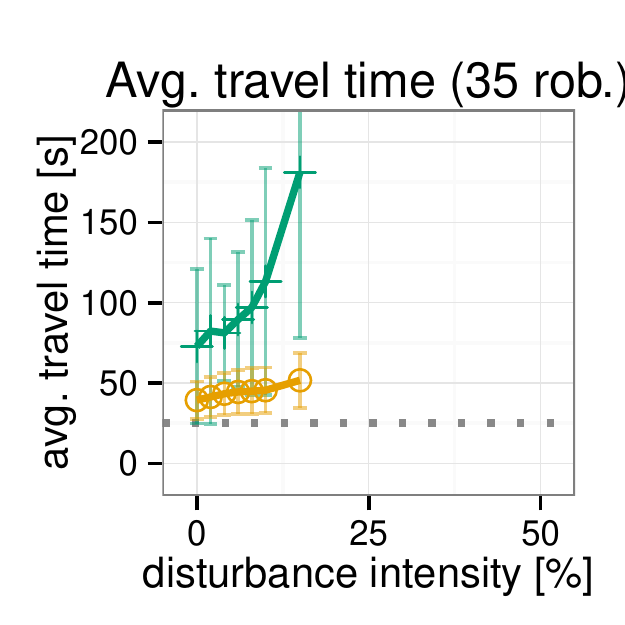} 
     & 
    \includegraphics[width=2.7cm]{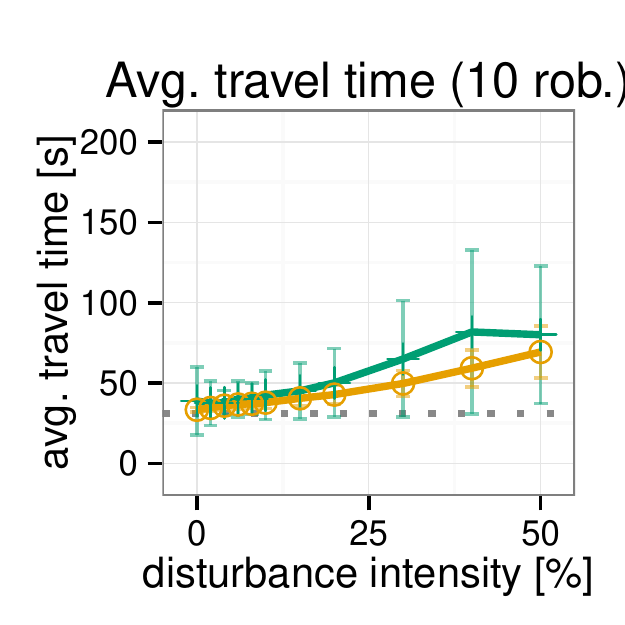}
    \includegraphics[width=2.7cm]{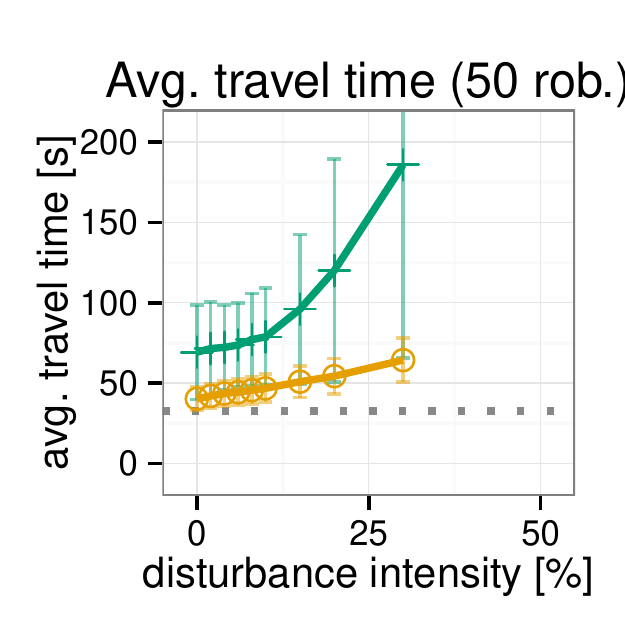} \\  
    \end{tabular}
    \includegraphics[scale=0.5]{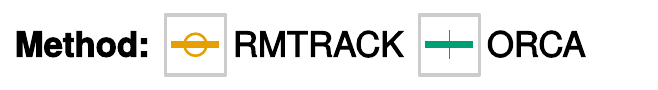} 
  \end{center}
\caption{Experimental comparison of ORCA and RMTRACK. The average travel time is computed if there is at least five instances for given number of robots and disturbance intensity successfuly solved by both evaluated algorithms. The dotted line representes the average travel time from origin to destination assuming that no disturbance and no need for coordination with other robots. The bars represent standard deviation of the difference between the travel time under the evaluated algorithm and the average travel time from origin to destination ignoring collisions with other robots and disiturbance.} \label{fig:orca-comparison}
\end{figure*}

In this section we will discuss the results of experimental comparison of RMTRACK approach against the baseline liveness-preserving method ALLSTOP and a reactive method ORCA using multi-robot simulation.

\subsubsection*{Experiment Setup}
The comparison was performed in three environments: Empty hall, Office corridor and Warehouse as shown in Figure~\ref{fig:environments}. 
A single problem instance in one of the environments consists of $n$ robots attempting to move from randomly generated origins to randomly generated destinations. 
We first find collision-free multi-robot trajectories from the origins to the destinations and then let each robot follow the given trajectory while randomly disturbing its advancement with specified intensity. More precisely, every second we decide with probablity corresponding to the disturbance intensity whether during the following second the robot will be prevented from moving.

To ensure that the initial coordinated trajectories can be found in a reliable and tractable fashion, the test environments satisfy so-called well-formed infrastructure property\footnote{In a well-formed infrastructure a start and destination of each robot is constrained to lie at a position where it does not completely prevent other robots from reaching their goals - most man-made infrastructures, e.g. a national road network system,  satisfy the property.}~\cite{cap_2015_b}. This allowed us to use revised prioritized planning approach~\cite{cap_2015_b} to efficiently find the initial trajectories for the robots to follow.   

In Empty hall and Warehouse environment we generated 10 instances with 10 robots and 10 instances with 50 robots; In Office corridor environment we generated 10 instances with 10 robots and 10 instances with 35 robots. Note that a single instance represents a specific assignment of origins and destinations to the robots.

An illustrative video, the source codes of all algorithms and benchmark instances are available at http://agents.fel.cvut.cz/~cap/rmtrack/.

\subsection*{Comparision of RMTRACK and ALLSTOP}
First, we compare two liveness-preserving control laws to handle disturbances: 1) ALL-STOP: the baseline law that makes entire multi-robot team stop whenever a single robot is disturbed and 2) RMTRACK: the law proposed in Section~\ref{sec:control-law}.
For each instance and disturbance intensity ranging from 0\,\% to 50\,\%, we run both algorithms and measured the time it took for each robot to reach its destination. 

In order to isolate the effect of each control law on the travel time from the effect of disturbances and the effect of the quality of the initial plan, we compute a lower-bound on the travel time of each robot assuming fixed disturbance and fixed initial plan. 
Such a lower bound is obtained by simulating the robot such that the inter-robot collisions are ignored and thus all robots always command to proceed at maximum advancement rate along their initial trajectory. Then the average advancement rate of the robot and consequently the travel time is affected solely by disturbances.
In fact, for uniformly distributed random disturbance with equal intensity $q$ for all robots, this corresponds to robots advancing on expectation at $1-q$ fraction of the original advancement rate $1$. 
Thus, the lower bound on expected travel time under disturbance intensity $q$ can be also computed as 
$ E(t_f)/(1-q)$, where $E(t_f)$ denotes expected travel time in the absence of disturbance.

It is easy to see that this lower bound represents the best possible travel time that can be achieved by RMTRACK, for instance, the travel time when the paths of the robots do not overlap. On the other hand, it is not difficult to construct a combination of problem instance and disturbances for which the behavior of RMTRACK degenerates to that of ALLSTOP. 
Curiously, since ALLSTOP proceeds only when none of the robots is disturbed, which for uniformly distributed disturbance with intensity $q$ at each robot happens with probability $(1-q)^n$, the expected travel time for ALLSTOP strategy can be consequently computed as $E(t_f)/(1-q)^n$, where $E(t_f)$ is again the expected travel time without disturbance and $n$ is the number of robots in the system. 

Consequently, we expect the average traveltime under RMTRACK strategy to be bounded from below by the lower-bound travel time and by the ALLSTOP travel time from above. 
The actual travel time under RMTRACK will then depend on the "interdependency" of initial trajectories and the level of disturbance. Given these two bounds, an interesting question is how will RMTRACK strategy perform in characteristic real-world environments. 

The results of performance comparison of RMTRACK with respect to ALLSTOP and the lower bound travel time for the three test environments are shown in Figure~\ref{fig:allstop-comparison}. We can see that consistently over all test environments and for different numbers of robots, the baseline strategy ALLSTOP quickly becomes impractical when the disturbance intensity is high. In contrast, the average travel time under RMTRACK remains reasonable even for high intensities of disturbance. Further, it is encouraging that for all three environments we tested on, the average travel time under RMTRACK remains close to the lower-bound travel time.

\subsection*{Comparision of RMTRACK and ORCA}

Next, we compared RMTRACK strategy with a reactive collision-avoidance technique ORCA~\cite{vanBerg2011_ORCA}, which is a characteristic representative of a family of popular collision avoidance algorithms based on the reciprocal velocity obstacle paradigm. 
Given the current velocities of all robots in the neighborhood and the desired velocity vector, it attempts to compute the closest velocity vector to the current desired velocity that does not lead to future collision with other robots, assuming that they will continue moving at their current velocity. 
In our implementation, the desired velocity at each time instance follows the shortest path to destination. 
For each instance we run ORCA and RMTRACK techniques for different disturbance intensities ranging from 0\,\% to 50\,\%.
During the experiment, we often witnessed ORCA leading robots to dead-lock situations during which the robots either moved at extremely slow velocities or even stopped completely. 
Therefore, if the robots failed to reach their destination within 10 minutes\footnote{average travel time between origin and destination ignoring collisions and without disturbance is around 25 second}, we considered the run as failed.

Figure~\ref{fig:orca-comparison} summarizes the results of the comparison. We can see that the success rate of ORCA deteriorates with increasing disturbance intensity. 
This is perhaps surprising since reactive methods are believed to be particularly well suited for unpredictable environments. 
Among of the reasons behind this phenomena seems to be that the reciprocal reactive algorithms rely on all robots executing the same algorithm and consequently on "splitting" the collision avoidance effort. 
This assumption is however violated if one of the robots is disturbed and does not execute the velocity command that the algorithm computed.
The plots in the bottom row show comparison of performance of RMTRACK and ORCA.  We can see that even when ORCA solves a given instance, the expected travel time for a robot is on expectation longer, especially so in cluttered environments and for high disturbance intensities.

\section{Conclusion}
\label{sec:conclusion}

The ability to guarantee safe and dead-lock free motion coordination for autonomous multi-robot systems in unpredictable environments is a standing challenge. Existing multi-robot trajectory planning techniques can provide guarantees on safety and liveness, but only if the computed trajectories are executed by all robots precisely. In this work, we have shown how to maintain guaranteed motion coordination under delay-inducing disturbances, i.e. in systems where the robots are able to follow the path precisely, but they can be temporarily stopped or delayed while executing the planned trajectory. These assumptions apply for example to autonomous intra-logistics systems in shared human-robot environments in which the robot is able to follow the preplanned path, but it is required to yield to all humans crossing its path. We have provided a formal proof that the control rule avoids inter-robot collisions and, more importantly, it preserves liveness property, which means that the robots are guaranteed to reach their goal positions without engaging in deadlocks. The method has been shown to be both more reliable and more efficient than the existing techniques used currently for coordination in multi-robot systems.

\bibliographystyle{plain}
\bibliography{references}

\end{document}